\documentclass[journal,twoside,web]{ieeecolor}

\usepackage{generic}
\usepackage{cite}
\usepackage{amsmath,amssymb,amsfonts}
\usepackage{algorithmic}
\usepackage{graphicx}
\usepackage{textcomp}

\usepackage{graphicx}
\usepackage{epsfig} 
\usepackage{cite}
\usepackage{lcsys}

\newtheorem{theorem}{Theorem}[section]
\newtheorem{lemma}[theorem]{Lemma}
\newtheorem{proposition}[theorem]{Proposition}

\newtheorem{definition}[theorem]{Definition}
\newtheorem{remark}[theorem]{Remark}
\newtheorem{assumption}[theorem]{Assumption}

\usepackage{mathtools}
\usepackage{bm}
\usepackage{siunitx}
\usepackage{subcaption}

\begin{document}

\def\BibTeX{{\rm B\kern-.05em{\sc i\kern-.025em b}\kern-.08em
    T\kern-.1667em\lower.7ex\hbox{E}\kern-.125emX}}
\markboth{\journalname, VOL. XX, NO. XX, XXXX 2017}
{Author \MakeLowercase{\textit{et al.}}: Preparation of Papers for IEEE Control Systems Letters (August 2022)}

\newcommand{\SE}{\mathrm{SE}(3)}
\newcommand{\SO}{\mathrm{SO(3)}}
\newcommand{\R}{\mathbb{R}}
\newcommand{\Vset}{\mathcal{V}}
\newcommand{\Eset}{\mathcal{E}}
\newcommand{\Aset}{\mathcal{A}}
\newcommand{\Oset}{\mathcal{O}}
\newcommand{\Iset}{\mathcal{I}}
\newcommand{\Gcal}{\mathcal{G}}
\newcommand{\Wcal}{\mathcal{W}}

\title{Multi-Robot Bearing-based Pose Estimation via Angle Rigidity*}

\author{
    J. Francisco Presenza$^{1}$,
    Leonardo J. Colombo$^{2}$,
    Ignacio Mas$^{3}$, \IEEEmembership{Senior Member, IEEE},
    and \\
    Juan I. Giribet$^{3}$, \IEEEmembership{Senior Member,~IEEE}
    \thanks{\copyright~2026 IEEE. Personal use of this material is permitted.
    Permission from IEEE must be obtained for all other uses, in any current or
    future media, including reprinting/republishing this material for advertising
    or promotional purposes, creating new collective works, for resale or
    redistribution to servers or lists, or reuse of any copyrighted component
    of this work in other works.}
    \thanks{*This work was supported by Universidad de Buenos Aires under Project UBACyT 20020220100053BA.  L. Colombo acknowledges financial support from Grant PID2022-137909NB-C21 funded by MCIN/AEI/ 10.13039/501100011033 and iRoboCity2030-CM, Robótica Inteligente para Ciudades Sostenibles (TEC-2024/TEC-62). 
    Corresponding author J.~Francisco Presenza.}
    \thanks{$^{1} $Institute of Engineering Technology and Sciences (CONICET-UBA), Buenos Aires, Argentina.
    {\tt\small jpresenza@fi.uba.ar}}%
    \thanks{$^{2}$ Centre for Automation and Robotics (CSIC-UPM), Madrid, Spain.
    {\tt\small leonardo.colombo@csic.es}}%
    \thanks{$^{3}$ Artificial Intelligence and Robotics Laboratory (CONICET-UdeSA), Victoria, Argentina.
    {\tt\small imas@udesa.ar, jgiribet@udesa.ar}}%
}

\maketitle
\thispagestyle{empty}

\begin{abstract}
This letter proposes a novel distributed pose estimator for multi-robot systems evolving on $\SE$. 
The robots' positions are estimated in $\R^3$ using angles computed from body-frame bearings, without requiring orientation knowledge.
The robots' orientations are then recovered in $\SO$ from the estimated positions, together with bearing and bearing-rate measurements.
The estimator accommodates directed sensing topologies and requires only \textit{infinitesimal angle rigidity} (IAR), thereby relaxing the requirement, common in bearing-based approaches, that every robot acquire at least two bearings.
Unlike existing angle-based schemes, the proposed method also estimates the robots' orientations.
We prove local uniform exponential stability of the observer, assuming that a subset of robots executes persistently exciting motions.
These theoretical results are corroborated through numerical simulations.
\end{abstract}


\section{Introduction}
\label{sec:introduction}

\IEEEPARstart{D}{istributed} pose estimation using only inter-agent measurements is a fundamental problem in multi-robot systems.
Over the last few years, extensive research has focused on \textit{bearing-based} methods, which exploit line-of-sight direction measurements between robots.
This interest is largely motivated by the ability to extract bearings from low-cost sensors (e.g., monocular cameras) with minimal processing.

Bearings are typically measured in each robot's body frame and transformed into a common frame for position estimation.
To do this, knowledge of the robot orientations is required.
However, absolute orientation estimation is often difficult in practice, since commonly used magnetic-field measurements are highly susceptible to disturbances and may be unavailable in some scenarios.
Alternatively, some strategies estimate the orientations from inter-agent relative rotation measurements \cite{Trinh2018CCTA,Li2020TCNS}.
Nevertheless, measuring relative rotations is non-trivial, as it often requires substantially more processing, calibration, or additional equipment (e.g., fiducial markers).

To avoid this, many works propose jointly estimating position and orientation directly from body-frame bearings.
Early approaches relied on the \textit{infinitesimal bearing rigidity} (IBR) of the sensing network to ensure that the poses are observable, although limited to $\mathbb{R}^3 \times \mathbb{S}^1$ \cite{Schiano2016IROS}. 
More recently, \cite{VanTran2020AUT,Cao2021AUT,Boughellaba2022CDC,Boughellaba2023CSL} operate on $\SE$ but still exhibit limited applicability, especially under sensing constraints which induce changes in the sensing topology.
This is partly because these methods depend on carefully designed topologies, involving \textit{(a)} cascaded interaction graphs, i.e., robot $i$ only uses information from robots $j < i$, and \textit{(b)} reciprocal bearing measurements, i.e., $i$ observes $j$ if and only if $j$ observes $i$.
Additionally, a key limitation of existing bearing-based methods, including \cite{Schiano2016IROS}--\cite{Boughellaba2023CSL}, is that they require a minimum number of measured bearings per robot: one in $\mathbb{S}^1$ and two in $\SO$.
This undermines the use of bearing sensors (e.g., cameras) for other tasks (e.g., mapping, target tracking).

More recently, \textit{angle rigidity theory} has emerged as a promising avenue for multi-robot localization \cite{Jing2021TAC,Chen2022AUTb,Chen2022TSP}.
Angle measurements, which are obtained from the dot product of measured bearings, are independent of the robots' orientations.
Hence, positions can be estimated without orientation knowledge.
However, existing angle-based localization approaches still face practical limitations.
For example, most are only applicable to \textit{undirected} graphs and require special graph topologies (e.g., tetrahedral).
Moreover, current angle-based methods do not provide orientation estimation, which is required in many robotic applications (e.g., map fusion).

In this work, we introduce a distributed observer based on body-frame bearing measurements to estimate the time-varying poses of a multi-robot system in $\SE$.
{For position estimation, the observer employs angles computed from bearing pairs. The sensing graph may be \textit{directed}, while the realization in $\R^3$ defined by the robots' positions must satisfy  \textit{infinitesimal angle rigidity} (IAR).
As shown in \cite{Presenza2026arXiv}, IAR for directed graphs allows sensing configurations in which some robots acquire only one bearing, or none at all, representing an advantage of our approach with respect to existing bearing-based methods.}
For orientation estimation, the proposed observer uses the position estimates together with bearing and bearing-rate measurements.
To the best of our knowledge, this is the first bearing-based pose estimator that does not require any of the following: \textit{(a)} reciprocal measurements (graphs can be directed), \textit{(b)} a minimum number of measured bearings per robot (some may not measure bearings at all), and \textit{(c)} additional sensing capabilities (e.g., inter-agent rotations).
In this work, we study the stability of the proposed observer for fixed sensing topologies.
Its formulation based on rigidity theory makes it suitable for integration with controllers that preserve rigidity under topology changes \cite{Presenza2026arXiv}. A formal stability analysis of such integration is left for future work.

This letter is organized as follows. Section \ref{sec:preliminaries} presents basic definitions and an overview of angle rigidity theory. In Section \ref{sec:observer}, we introduce the design of the distributed pose estimator, while the stability analysis is presented in \ref{sec:stability}.
Simulation results are presented in Section \ref{sec:simulations}, and Section \ref{sec:conclusions} provides final conclusions and directions for future work.

\section{Preliminaries}
\label{sec:preliminaries}

\subsection{Notation}

Vectors are written as $x \in \R^n$ and matrices  as $A \in \R^{m \times n}$.
The standard basis is $\{\mathfrak{e}_i\}_{i=1}^n \subset \R^n$.
The vector of all ones is $\bm{1}_n \in \R^n$ and the identity matrix is $\bm{I}_n \in \R^{n \times n}$.
The null space is $\mathrm{Null}(\cdot)$ and the Kronecker product is $\otimes $.
The unit $n$-sphere is $\mathbb{S}^{n-1}\!=\!\{v\!\in\!\R^n : \|v\|\!=\!1\}$ and $\Pi : \mathbb{S}^2 \to \R^{3 \times 3}$ satisfies $\Pi(x) \coloneqq \bm{I}_3 - x x^\top$.
For a $n \times n$ symmetric matrix, we index its eigenvalues in non-decreasing order: $\lambda_1(\cdot) \leq \ldots \leq \lambda_n(\cdot)$.
Let $\SO$ denote the special orthogonal group and $\SE = \R^3\rtimes \SO$.
Let $S:\R^3\to\R^{3 \times 3}$ be the skew-symmetric-valued map $S(x)y=x\times y$ for all $x,y\in\R^3$ where $\times$ denotes the cross product.
A column stack of $n$ vectors (or matrices) of equal dimensions indexed by $\mathcal{S} = \{1, \ldots, n\}$ is represented as $\bm{x} = [x_i]_{i \in \mathcal{S}}$ (or $\bm{A} = [A_i]_{i \in \mathcal{S}}$).
If each $x_i \in \R^3$, we let $\bm{S}(\bm{x}) \coloneqq [S(x_i)]_{i \in \mathcal{S}}$.
If the dimensions of $\bm{A} = [A_i]_{i \in \mathcal{S}}$ and $\bm{B} = [B_i]_{i \in \mathcal{S}}$ are appropriate, then $\bm{A} \circ\bm{B} = [A_i B_i]_{i \in \mathcal{S}}$ denotes the block-wise product.

\subsection{Robot and Network Models}

Let $(p_i, R_i)$ represent the pose of robot $i$, with position $p_i \in \R^3$ and orientation $R_i \in \SO$.
Let $\Vset = \{1, \ldots, N\}$ denote the robot indices, $N \geq 3$.
Then, $(\bm{p}, \bm{R})$, with $\bm{p} = [p_i]_{i \in \Vset} \in \R^{3N}$ and $\bm{R} = [R_i]_{i \in \Vset} \in \SO^N$, represents the joint configuration.
We consider first-order robot kinematics $$\dot{p}_i = R_i u_i, \quad \dot{R}_i = R_i S(\omega_i),$$ where $u_i,\omega_i \in \R^3$ denote the body-frame control actions of $i \in \Vset$. The sensing network is modeled by a \textit{directed} graph $\Gcal = (\Vset, \Eset)$ with vertex set $\Vset$ and edge set $\Eset \subseteq \Vset \times \Vset$.
The set of (outgoing) neighbors of $i$ is $\Oset_i \coloneqq \{j : (i, j) \in \Eset\}$.
Here, $j \in \Oset_i$ indicates that robot $i$ measures $\beta_{ij} \coloneqq R_i^\top (p_j-p_i)/\|p_j - p_i\|$.
In this work, we do not require reciprocal bearing measurements, i.e., $j \in \Oset_i$ does not imply $i \in \Oset_j$.
The (cosine of the) angle measurement obtained when $\{j, k\} \subseteq \Oset_i$ is $\alpha_{ijk} \coloneqq \beta_{ij}^\top \beta_{ik}$.
We define $\Aset \coloneqq \{(i, j, k) : \{j, k\} \subseteq \Oset_i, j < k\}$ to be the set of angle indices generated by every pair of bearings observed by the same robot. Hence $\Aset$ is determined by the directed edge set $\Eset$, reflecting how angles are measured in practice using camera-like sensors.
In addition to bearing measurements, we assume that robots are able to communicate over a graph $\overline{\Gcal}$ taken as the undirected counterpart of $\Gcal$.
The infinitesimal angle rigidity of $(\Gcal, \bm{p})$ (Assumption \ref{ass:uniform_IAR}) ensures that $\overline{\Gcal}$ is connected.

Now, consider an \textit{anchor robot} $a \in \Vset$ and express the robots' poses in a reference frame attached to $a$, $$\bm{q}\coloneqq[q_i]_{i\in\Vset}, \bm{Q}\coloneqq[Q_i]_{i\in\Vset}, \ q_i \coloneqq R_a^\top(p_i-p_a), Q_i \coloneqq R_a^\top R_i.$$
We assume that $a$ can measure both the bearing and the distance to two robots $b,c\in\Vset$, which eliminates the similarity ambiguity that arises in angle-based localization; see Subsection \ref{sec:angle_rigidity}.
It holds that,
\begin{align}
    \beta_{ij} &= Q_i^\top \mathfrak{q}_{ij}, \ \mathfrak{q}_{ij}\!\coloneqq\!q_{ij}/\|q_{ij}\|, \ q_{ij}\!\coloneqq q_j-q_i,
    \label{eq:bearing_definition} \\ 
    \dot{q}_i &= Q_i u_i - u_a - S(\omega_a) q_i,
    \
    \dot{Q}_i = Q_i S(\omega_i - Q_i^\top \omega_a).
    \label{eq:robot_dynamic_model_a}
\end{align}
{In \eqref{eq:robot_dynamic_model_a}, $u_a,\omega_a$ are the anchor's control actions, assumed to be disseminated via a leader-follower consensus estimator~\cite{Franchi2016TCNS}.}

\subsection{Angle Rigidity Theory}
\label{sec:angle_rigidity}
Consider a \textit{framework} $(\Gcal, \bm{q})$ and let $\mathcal A$ be the set of angle indices.
Define the \textit{angle function} as $\bm{\alpha}_{\Gcal} : \R^{3N} \to \R^{|\Aset|}$ such that $\bm{\alpha}_{\Gcal}(\bm{q}) = [\alpha_{ijk}]_{(i, j, k) \in \Aset}$.
Let the \textit{angle rigidity matrix} $A_{\Gcal}(\bm{q}) \in \R^{|\Aset| \times 3N}$ be the differential of $\bm{\alpha}_{\Gcal}$, i.e., $\delta \bm{\alpha}_{\Gcal} = A_{\Gcal}(\bm{q}) \delta \bm{q}$.
To make the structure explicit, consider the row $A_{ijk}(\bm{q}) \in \R^{1 \times 3N}$ corresponding to $(i, j, k) \in \Aset$,
\begin{equation}
    A_{ijk} = [\bm{0} \ \underset{(\text{block} \, i)}{-(\chi_{ijk}+\chi_{ikj})^\top} \ \bm{0} \ \underset{(\text{block} \, j)}{\chi_{ijk}^\top} \ \bm{0} \ \underset{(\text{block} \, k)}{\chi_{ikj}^\top} \ \bm{0}]
\label{eq:angle_rigidity_matrix}
\end{equation}
where $\chi_{ijk} \coloneqq \Pi(\mathfrak{q}_{ij}) \mathfrak{q}_{ik} / \|q_{ij}\|$.
Consider the subspace
\begin{equation}
    \!\mathcal{T}(\bm{q})\!\coloneqq\!\{\kappa\bm{q}+\bm{S}(\bm{1}_{\!N}\otimes\omega)\circ\bm{q}\,+\bm{1}_N\otimes x\!: \kappa\!\in\!\R, \omega,x\!\in\!\R^3\},
\label{eq:angle_inf_trivial}
\end{equation}
i.e., tangent vectors at $\bm{q}$ induced by similarity transformations: translations, rotations, and scalings.
Provided that $\{q_i\}_{i \in \Vset} \subset \R^3$ is not contained in a hyperplane, $\operatorname{dim} \mathcal{T}(\bm{q})\!=\!7$.
Moreover, it holds that $\mathcal{T}(\bm{q}) \subseteq \operatorname{Null}(A_{\Gcal}(\bm{q}))$.
\begin{definition}
    A framework $(\Gcal, \bm{q})$ is said to be \textit{infinitesimally angle rigid} (IAR) if $\operatorname{Null}(A_{\Gcal}(\bm{q})) = \mathcal{T}(\bm{q})$.
    Equivalently, if $\lambda_8(\bm{L}_{\Gcal}(\bm{q})) > 0$ where $\bm{L}_{\Gcal}(\bm{q}) \coloneqq A_{\Gcal}^\top(\bm{q}) A_{\Gcal}(\bm{q})$.
    \label{def:inf_ang_rigid}
\end{definition}
The \textit{angle rigidity eigenvalue}, $\lambda_8(\bm{L}_{\Gcal}(\bm{q}))$, is important for angle-based localization, as it enables the recovery of the positions up to a similarity transformation.
This is established by Proposition \ref{prop:inf_angle_rigidity}, a classical result in rigidity theory that follows from applying the inverse function theorem to $\bm{\alpha}_{\Gcal}$.
\begin{proposition}
    Let $(\Gcal, \bm{q})$ be IAR. Then, there exists an open neighborhood $\mathcal{U}$ of $\bm{q}$ such that every $\bm{q}' \in \bm{\alpha}_{\Gcal}^{-1}(\bm{\alpha}_{\Gcal}(\bm{q})) \cap \mathcal{U}$
    is similar to $\bm{q}$.
    \label{prop:inf_angle_rigidity}
\end{proposition}

\section{Observer Design}
\label{sec:observer}

Motivated by the structure of the proposed observer, we split the robots into two groups.
First, the \textit{sensing robots}: those that obtain at least two bearings and can hence compute angles, $\Vset^s\!\coloneqq\!\{i\!\in\!\Vset\!: |\mathcal{O}_i|\!\geq\!2\}$.
Second, the \textit{free robots}: those that obtain one or no bearings, $\Vset^f \coloneqq \Vset \setminus \Vset^s$.
It follows that the anchor $a \in \Vset^s$.
The observer is designed as follows.
First, it recovers the positions of all $i \in \Vset$ in the anchor $a$'s reference frame using the angles computed from bearing measurements.
Second, it uses bearing measurements to recover the orientations of robots $i \in \Vset^s$.
Third, it exploits bearing-rate measurements and suitable robot motions to recover the orientations of robots $i \in \Vset^f$.
For the remainder of the paper, for any stacked vector $\bm{x} = [x_i]_{i\in\Vset}$, we denote by $\bm{x}^s$ and $\bm{x}^f$ the subvectors containing the entries associated with sensing robots and free robots, respectively.

\subsection{Assumptions}
We state the assumptions required to establish the stability of the proposed observer (Theorem~\ref{thm:full_cascade_lues}). 
First, Assumption \ref{ass:separation} states that collisions between robots are avoided.
Assumption \ref{ass:uniform_IAR} is an observability condition (see Subsection~\ref{sec:position_correction}) and can be enforced in practice with a rigidity maintenance controller \cite{Presenza2026arXiv}.
Assumption \ref{ass:noncollinear_outgoing_bearings} ensures that the outgoing bearings of each sensing robot are not collinear, which fails to hold only on a measure-zero set. If required, it could be incorporated as an additional objective into controller \cite{Presenza2026arXiv}.
Assumption \ref{ass:dynamics_regularity} requires sufficiently regular control actions and bounded inter-robot distances, both practical requirements.
Assumption \ref{ass:excitation} requires the velocity of each free robot $\dot{p}_i$ to be persistently exciting (PE), i.e., not confined to a fixed direction over each time window of length $T$.
Finally, Assumption \ref{ass:exact_consensus} idealizes the dissemination layer. 
In practice, when $\dot{u}_a$ and $\dot{\omega}_a$ are bounded, leader-follower consensus \cite{Franchi2016TCNS} yields an ultimate bound on the consensus error that decreases as the consensus gain increases.
In the ideal case of zero consensus error, the  observer is locally uniformly exponentially stable (Theorem~\ref{thm:full_cascade_lues}).
Hence, standard robustness arguments suggest that sufficiently small errors lead to small pose-estimation errors.
A formal analysis of the coupled consensus-observer dynamics is left for future work.
\begin{assumption}[Uniform inter-agent separation]
    There exists $\ell > 0$ such that $\|q_{ij}(t)\| > \ell$, $\forall t \geq 0$.
    \label{ass:separation}
\end{assumption}
\begin{assumption}[Uniform IAR]
    There exists $\ell>0$ such that $\lambda_8(\bm{L}_{\Gcal}(\bm{q}(t))) > \ell$, $\forall t \geq 0$.
    \label{ass:uniform_IAR}
\end{assumption}
\begin{assumption}[Uniform non-collinearity of bearings]
    There exists $\ell\!>\!0$ such that:
    \begin{enumerate}
        \item $\Pi(\beta_{ab}(t)) + \Pi(\beta_{ac}(t)) \succeq\ell \bm{I}$, $\forall t \geq 0$.
        \item For each $i\!\in\!\Vset^s\setminus\{a\}$, $\textstyle\sum_{j \in \Oset_i}\!\Pi(\beta_{ij}(t))\succeq\ell \bm{I}$, $\forall t \geq 0$.
    \end{enumerate}
    \label{ass:noncollinear_outgoing_bearings}
\end{assumption}
\begin{assumption}[Regularity]
    For each $i \in \Vset$, the signals $q_i(t), u_i(t), \omega_i(t)$ are sufficiently regular and there exists $\mu_q, \mu_{u}, \mu_{\omega} >0$ such that $\|q_i(t)\| < \mu_q$, $\|u_i(t)\| < \mu_{u}$, $\|\omega_i(t)\| < \mu_{\omega}$,  $\forall t \geq 0$.
    \label{ass:dynamics_regularity}
\end{assumption}
\begin{assumption}[Persistent excitation of the free robots]
\label{ass:excitation}
    There exist constants $T>0$ and $\ell>0$ such that for each $i \in \Vset^f$, $\textstyle\int_t^{t+T} {S(\dot{p}_i(\tau))^\top S(\dot{p}_i(\tau))} d\tau\succeq \ell \bm{I}$, $\forall t \geq 0$.
\end{assumption}
\begin{assumption}
    Consensus errors are neglected: each robot has access to the anchor inputs $u_a(t), \omega_a(t)$.
    \label{ass:exact_consensus}
\end{assumption}

\subsection{Position Correction}\label{sec:position_correction}

We begin by presenting a gradient-based position estimator based on angle measurements.
Let $\hat{\bm{q}}\!=\![\hat{q}_i]_{i \in \Vset}$ be the estimated joint position, and let $\tilde{\bm{q}}\!\coloneqq\!\hat{\bm{q}}-\bm{q}$ be the estimation error.
Consider a fixed configuration $\bm{q}$.
Let $\tilde{\bm{\alpha}}_{\Gcal}(\hat{\bm{q}}) = [\tilde{a}_{ijk}]_{(i,j,k)\in\mathcal A} \coloneqq \bm{\alpha}_{\Gcal}(\hat{\bm{q}}) - \bm{\alpha}_{\Gcal}(\bm{q})$ denote the angle-estimation errors, where $\bm{\alpha}_{\Gcal}(\bm{q})$ represents the measured angles.
Then, for $\kappa_s>0$, we define the cost function
\begin{equation}
    \mathcal{L}(\hat{\bm{q}}) \coloneqq
    \tfrac{1}{2} \textstyle\sum_{(i, j, k) \in \Aset} \tilde{\alpha}_{ijk}^2 + \frac{\kappa_s}{2}
    (\|\tilde{q}_a\|^2 + \|\tilde{q}_b\|^2 + \|\tilde{q}_c\|^2).
\label{eq:position_cost}
\end{equation}
Here, $\tilde{q}_a=\hat{q}_a$, $\tilde{q}_b = \hat{q}_b - q_b$, and $\tilde{q}_c = \hat{q}_c - q_c$, where $q_b, q_c$ are available from the anchor's
bearing and distance measurements.
From Proposition \ref{prop:inf_angle_rigidity}, IAR implies that, locally, all minimizers of the first term of \eqref{eq:position_cost} are similar to $\bm{q}$.
From Assumptions \ref{ass:separation}  and \ref{ass:noncollinear_outgoing_bearings}, $q_b$ and $q_c$ are non-zero and non-collinear; hence, the second term \eqref{eq:position_cost} removes the similarity ambiguity (see Lemma \ref{lem:positive_hessian}). 
It follows that $\hat{\bm{q}} = \bm{q}$ is an isolated global minimizer of $\mathcal{L}$.
The gradient of $\mathcal{L}$ with respect to the $i$-th estimate is
\begin{equation}
\begin{split}
    &\!\!\nabla_{\hat{q}_i} \mathcal{L}(\hat{\bm{q}})\!= - \textstyle\sum_{j, k:\,\{j, k\} \subseteq \mathcal{O}_i} \tilde{\alpha}_{ijk} (\hat{\chi}_{ijk} + \hat{\chi}_{ikj}) \\
    &\!\!+ \textstyle\sum_{j, k:\,\{i, k\} \subseteq \mathcal{O}_j} \tilde{\alpha}_{jik} \hat{\chi}_{jik}
    + \kappa_s (\delta_{ia} \tilde{q}_a +  \delta_{ib} \tilde{q}_b + \delta_{ic} \tilde{q}_c),
\end{split}
\label{eq:position_gradient_i}
\end{equation}
where $\hat{\chi}_{ijk} \coloneqq \Pi(\hat{\mathfrak{q}}_{ij}) \hat{\mathfrak{q}}_{ik} / \|\hat{q}_{ij}\|$ with $\hat{\mathfrak{q}}_{ij} \coloneqq \hat{q}_{ij} / \|\hat{q}_{ij}\|$, and $\delta_{i j}$ is Kronecker's delta. 
\begin{remark}
    Robot $i$ can compute~\eqref{eq:position_gradient_i} in a distributed manner by receiving $\hat{q}_j$ from each $j\!\in\!\Oset_i$ and $\tilde{\alpha}_{jik} \hat{\chi}_{jik}$ from each $j$ such that $\{i, k\} \subseteq \Oset_j$, which are locally available to each corresponding $j$.
    \label{rmk:pos_est_dist}
\end{remark}

\begin{lemma}[Uniform positivity of the position Hessian]
\label{lem:positive_hessian}
Let $(\Gcal, \bm{q}(t)), \, t\!\geq\!0$ satisfy Assumptions \ref{ass:separation}--\ref{ass:dynamics_regularity}. Then the Hessian $\bm{H}(\bm{q}(t))$ of $\mathcal{L}$ is uniformly positive definite along $\bm{q}(t)$, i.e., there exists $\ell\!>\!0$ such that $\bm{H}(\bm{q}(t))\!\succeq\!\ell\bm{I}$, $\forall t\!\geq\!0$.
\end{lemma}
\begin{proof}
    Let $\mathcal{Q} \coloneqq \operatorname{cl} \{\bm{q}(t) \in \R^{3N} :t \geq 0\}$ be the closure of the position trajectory, which is compact since $\bm{q}(t)$ is bounded (Assumption \ref{ass:dynamics_regularity}).
    
    Assumptions \ref{ass:separation}--\ref{ass:noncollinear_outgoing_bearings} hold in $\mathcal{Q}$, due to the continuity of $\|\cdot\|$ and $\lambda_8(\cdot)$.
    That is, for each $\bm{q} \in \mathcal{Q}$, $q_b, q_c$ are non-zero and non-collinear, and $\lambda_8(\bm{L}_{\Gcal}(\bm{q}))>0$ thus $\operatorname{Null}(\bm{L}_{\Gcal}(\bm{q})) = \mathcal{T}(\bm{q})$.
    The Hessian is
    \begin{equation*}
        \bm{H}(\bm{q}) = \bm{L}_{\Gcal}(\bm{q}) + \kappa_s\left(\mathfrak{e}_a\mathfrak{e}_a^\top+\mathfrak{e}_b\mathfrak{e}_b^\top+\mathfrak{e}_c\mathfrak{e}_c^\top\right)\otimes \bm{I}_3.
    \end{equation*}
    where $\mathfrak{e}_i\in\R^N$.
    Now, let $\bm{v}\!=\![v_i]_{i \in \Vset}\in\R^{3N}$ such that $\bm{v} = \bm{n}+\bm{r}$ where $\bm{n}\in\mathcal{T}(\bm{q})$ and $\bm{r}\in\mathcal{T}(\bm{q})^\perp$. Thus,
    \begin{align*}
        \bm{v}^\top &\bm{H}(\bm{q}) \bm{v} = \bm{r}^\top \bm{L}_{\Gcal}(\bm{q}) \bm{r} + \kappa_s(\|v_a\|^2 + \|v_b\|^2 + \|v_c\|^2) \\
        &\geq \lambda_8(\bm{L}_{\Gcal}(\bm{q})) \|\bm{r}\|^2 + 
        \kappa_s(\|v_a\|^2 + \|v_b\|^2 + \|v_c\|^2).
    \end{align*}
    Thus, if $\bm{v}^\top \bm{H}(\bm{q}) \bm{v} = 0$, then $v_a= v_b = v_c = 0$ and $\bm{r} = 0$.
    Hence, $\bm{v} \in \mathcal{T}(\bm{q})$ corresponds to a velocity generated by a similarity transformation that fixes $q_a$, $q_b$, and $q_c$.
    But the only similarity transformation that does so is the identity map; thus, $\bm{v} = 0$.
    Consequently, $\lambda_1(\bm{H}(\bm{q})) > 0$ for all $\bm{q} \in \mathcal{Q}$. Since $\mathcal{Q}$ is compact and $\lambda_1(\bm{H}(\cdot))$ is continuous, it attains a positive minimum, and uniformity follows.
\end{proof}

\subsection{Coordinate Velocity Reconstruction}
{In this subsection, we propose a method to reconstruct the robots' coordinate velocities, $\dot{\bm{q}}$.
To do this, we introduce the map $\bm{\eta}$ such that $\dot{\bm{q}} = \bm{\eta}(\bm{q}, \bm{Q}^s)$. For brevity, we omit the dependence of $\bm{\eta}$ on the control inputs $\bm{u} \coloneqq [u_i]_{i\in\Vset}$, $\bm{\omega} \coloneqq [\omega_i]_{i\in\Vset}$, as well as the bearings $\bm{\beta} \coloneqq [\beta_{ij}]_{(i, j)\in\Eset}$ and their time derivatives $\dot{\bm{\beta}}$.
The map $\bm{\eta}$ is designed to avoid using $\bm{Q}^f$ for two reasons.
First, it enables orientation-error correction of the free robots: comparing the reconstructed $\dot{\bm{q}}^f$ with $\bm{u}^f$ via \eqref{eq:robot_dynamic_model_a} reveals information about $\bm{Q}^f$.
Second, it allows us to include a feedforward term in the position part of estimator \eqref{eq:upstream}-\eqref{eq:downstream} without using $\bm{Q}^f$, see Remark \ref{rmk:cascade}.

For sensing robots, we adopt the standard reconstruction given by \eqref{eq:robot_dynamic_model_a}. Indeed, the coordinate velocity $\dot{q}_i$ and its estimate $\hat{\upsilon}_i$ can be reconstructed as
\begin{align}
   \dot{q}_i &= \eta^s_i(\bm{q}, \bm{Q}^s) \coloneqq Q_i u_i - u_a - S(\omega_a) q_i,
    \label{eq:eta_leader} \\
    \hat{\upsilon}_i &\coloneqq \eta^s_i(\hat{\bm{q}}, \hat{\bm{Q}}^s) = \hat{Q}_i u_i - u_a - S(\omega_a) \hat{q}_i.
    \label{eq:vel_leader}
\end{align}
For the free robots, first, define $\Iset^s_i \coloneqq \{j \in \Vset^s: (j, i) \in \Eset\}$, i.e., the set of sensing robots observing free robot $i$.
Lemma \ref{lem:noncollinear_ingoing_bearings} shows that each free robot receives at least two non-collinear bearing measurements from sensing robots, which guarantees the well-posedness of the reconstructed $\dot{q}_i$ \eqref{eq:eta_follower}.
\begin{lemma}[Uniform non-collinearity of the in-bearings]
    Let $(\Gcal, \bm{q}(t)), \, t \geq 0$ satisfy  Assumptions \ref{ass:separation}--\ref{ass:dynamics_regularity}. 
    Then, for each $i \in \Vset^f$, there exists $\ell>0$ such that
    \begin{equation}
        P_i(\bm{q}(t)) \coloneqq \textstyle\sum_{j \in \Iset^s_i} \Pi(\mathfrak{q}_{ji}(t)) \succeq\ell \bm{I},
        \qquad \forall t \geq 0.
        \label{eq:noncollinear_ingoing_bearings}
    \end{equation}
    \label{lem:noncollinear_ingoing_bearings}
\end{lemma}
\begin{proof}
    Define $\mathcal{Q} \coloneqq \operatorname{cl} \{\bm{q}(t) \in \R^{3N} :t \geq 0\}$.
    Assumptions \ref{ass:separation}--\ref{ass:noncollinear_outgoing_bearings} hold in $\mathcal{Q}$ by the same argument used in the proof of Lemma \ref{lem:positive_hessian}.
    
    First, we show that $P_i(\bm{q}) \succ 0$ for all $\bm{q} \in \mathcal{Q}$ and all $i \in \Vset^f$. 
    Fix $\bm{q} \in \mathcal{Q}$ and $i \in \Vset^f$ and suppose, by contradiction, that $P_i(\bm{q})$ is singular. 
    If $\Iset^s_i \neq \emptyset$, since $P_i(\bm{q})$ is singular, all bearings $\mathfrak{q}_{ji}$ are collinear.
    Thus, choose $j \in \Iset^s_i$ and let $\delta \bm{q} \coloneqq \mathfrak{e}_i \otimes \mathfrak{q}_{ji}$.
    Then, $A_{\Gcal}(\bm{q}) \delta \bm{q} = 0$ and $\delta \bm{q} \notin \mathcal{T}(\bm{q})$, see \eqref{eq:angle_rigidity_matrix} and \eqref{eq:angle_inf_trivial}, contradicting IAR.
    If $\Iset^s_i\!=\!\emptyset$, an analogous argument follows with $\delta \bm{q}\!\coloneqq \mathfrak{e}_i\!\otimes x \in\!\R^{3N}$ for any $x\!\neq\!0 \in\!\R^3$.

    From the implication proved above, $\lambda_1(P_i(\bm{q}))\!>\!0$ for each $\bm{q} \in \mathcal{Q}$.
    Since $\lambda_1(P_i(\cdot))$ is continuous and $\mathcal{Q}$ is compact, it attains a positive minimum, and uniformity follows.
\end{proof}

Now, we reconstruct $\dot{q}_i, \, i \in \Vset^f$. Differentiating $\mathfrak{q}_{ji} = Q_j \beta_{ji}$ and rearranging yields the linear equations on $\dot{q}_i$,
\begin{equation}
    \Pi(Q_j \beta_{ji}) (\dot{q}_i - \dot{q}_j) = \|q_{ij}\| (\dot{Q}_j \beta_{ji} + Q_j \dot{\beta}_{ji}), \ \forall j \in \Iset^s_i.
    \label{eq:linear_eqs}
\end{equation}
Proposition \ref{pro:linear_sys_solution} provides the unique solution to this system.

\begin{proposition}
    Let $(\Gcal,\bm{q}(t)),t\!\geq\!0$ satisfy Assumptions~\ref{ass:separation}--\ref{ass:dynamics_regularity}.
    Then, for each $i\in\Vset^f$, the unique solution of~\eqref{eq:linear_eqs} is
    \begin{equation}
    \begin{split}
        \!\!\dot{q}_i &= \eta^f_i(\bm{q}, \bm{Q}^s) \coloneqq\big(\textstyle\sum_{j \in \Iset^s_i} M_{ji}\big)^{-1}
        \textstyle\sum_{j \in \Iset^s_i} b_{ji}, \\
        \!\!M_{ji}&\!\coloneqq\!\Pi(Q_j \beta_{ji}), b_{ij}\!\coloneqq\!M_{ji}\dot{q}_j\!+\!\|q_{ji}\|(\dot{Q}_j\beta_{ji}\!+\!Q_j \dot{\beta}_{ji}).
    \end{split}
    \label{eq:eta_follower}
    \end{equation}
    \label{pro:linear_sys_solution}
\end{proposition}
\begin{proof}
    System \eqref{eq:linear_eqs} has the form $M_i\dot{q}_i = b_i$ where $M_i\!\coloneqq\![M_{ji}]_{j \in \Iset^s_i}$ and $b_i\!\coloneqq\![b_{ji}]_{j \in \Iset^s_i}$.
    This system is compatible since it is derived from the equality $\mathfrak{q}_{ji} = Q_j \beta_{ji}$.
    Also, $M_i^\top M_i = P_i\!\succ\!0$ from $M_{ji}^2\!=\!M_{ji}$ and \eqref{eq:noncollinear_ingoing_bearings}. Thus, \eqref{eq:eta_follower} is the unique solution of the normal equation $P_i \dot{q}_i = M_i^\top b_i$.
\end{proof}
\begin{remark}
    \eqref{eq:eta_follower} provides a least-squares solution that is exact in the noiseless case and is especially useful in the case of noisy bearing and bearing-rate measurements.
\end{remark}
Finally, the estimated coordinate velocity $\hat{\upsilon}_i, i \in \Vset^f$ is
\begin{equation}
    \begin{split}
        \!\!\hat{\upsilon}_i &\coloneqq \eta^f_i(\hat{\bm{q}}, \hat{\bm{Q}}^s) =\big(\textstyle\sum_{j \in \Iset^s_i} \hat{M}_{ji}\big)^{-1}
        \textstyle\sum_{j \in \Iset^s_i} \hat{b}_{ji}, \\
        \!\!\!\hat{M}_{ji}&\!\coloneqq\!\Pi(\hat{Q}_j \beta_{ji}), \hat{b}_{ij}\!\coloneqq\!\hat{M}_{ji}\hat{\upsilon}_j\!+\!\|\hat{q}_{ji}\|(\hat{\Psi}_j\beta_{ji}\!+\!\hat{Q}_j \dot{\beta}_{ji})
    \end{split}
    \label{eq:vel_follower}
\end{equation}
with $\hat{\upsilon}_j$ from \eqref{eq:vel_leader} and $\hat{\Psi}_j \coloneqq \hat{Q}_j S(\omega_j - \hat{Q}_j^\top \omega_a)$ from \eqref{eq:robot_dynamic_model_a}.
\begin{remark}
    Robot $i \in \Vset^f$ can compute \eqref{eq:vel_follower} by receiving $\hat{M}_{ji}, \hat{b}_{ji}$ from $j \in \Iset_i^s$, which are locally available to $j$.
    \label{rmk:vel_est_dist}
\end{remark}
\begin{remark}
    In practice, robot $j \in \Iset_i^s$ can estimate $\dot{\beta}_{ji}$ from consecutive camera measurements using feature tracking or optical flow, followed by suitable filtering.
\end{remark}
\begin{remark}
    Velocities $u_a, \omega_a$ appear because all poses are expressed in the moving anchor frame. Their role is to compensate for $a$'s motion in the feedforward terms.
\end{remark}

\subsection{Orientation Correction}

\subsubsection{Sensing robots}
For $i \in \Vset^s$, it is possible to estimate $\hat{Q}_i$ by exploiting the relationship $\beta_{ij}^\top Q_i^\top q_{ij} = \|q_{ij}\|$ \eqref{eq:bearing_definition}, $\forall j \in \Oset_i$.
Indeed, consider the following cost function and the $\SO$ gradient-based update for $i \in \Vset^s$,
\begin{equation}
\begin{split}
    &\!\!\mathcal{J}^s_i(\hat{\bm{q}}, \hat{Q}_i) \coloneqq \textstyle\sum_{j \in \Oset_i} \big(\|\hat{q}_{ij}\| - \beta_{ij}^\top \hat{Q}_i^\top \hat{q}_{ij}\big), 
    \\
    &\!\!\dot{\hat{Q}}_i\!=\!-\nabla_{\!\hat{Q}_i}\!\mathcal{J}^s_i\!=\!\hat{Q}_i S(\Omega^s_i), \, \Omega^s_i\!\coloneqq\!\textstyle\sum_{j \in \Oset_i} S(\beta_{ij}) \hat{Q}_i^\top \hat{q}_{ij}.
\end{split}    
\label{eq:grad_ori_leaders}
\end{equation}
Note that $\mathcal{J}^s_i \geq 0$ and $\mathcal{J}^s_i(\bm{q}, Q_i) = 0$.

\subsubsection{Free robots}
\label{sec:orientation_followers}

For $i\!\in\!\Vset^f$, \eqref{eq:grad_ori_leaders} is not appropriate since it requires $|\Oset_i| \geq 2$ for full error correction.
However, $Q_i$ can be estimated by comparing the body-frame control action $u_i$ with the estimated coordinate velocity $\hat{\upsilon}_i$ \eqref{eq:vel_follower}.
To do this, define $\gamma_i \coloneqq Q_i u_i$ and its reconstruction from estimates $\hat{\gamma}_i \coloneqq \hat{\upsilon}_i + u_a + S(\omega_a) \hat{q}_i$ \eqref{eq:robot_dynamic_model_a}.
We propose the following cost function and the $\SO$ gradient-based update,
\begin{equation}
\begin{split}
    &\!\mathcal{J}^f_i(\hat{\upsilon}_i, \hat{Q}_i) \coloneqq \|u_i\| \|\hat{\gamma}_i\| - u_i^\top \hat{Q}_i^\top \hat{\gamma}_i, 
    \\
    &\dot{\hat{Q}}_i =-\nabla_{\hat{Q}_i}\!\mathcal{J}^f_i = \hat{Q}_i S(\Omega^f_i), \ \Omega^f_i\!\coloneqq\!S(u_i) \hat{Q}_i^\top \hat{\gamma}_i.
\end{split}
\label{eq:grad_ori_followers}
\end{equation}
Note that $\mathcal{J}^f_i\!\geq0$ and $\mathcal{J}^f_i(\dot{q}_i, Q_i)\!=0$.

\subsection{Full Pose Observer}
\label{sec:observer_formulation}

Building on the preceding analysis, we introduce the following observer.
Set $\kappa_q, \kappa_Q>0$, then
\begin{align}
    \begin{split}
        \dot{\hat{q}}_i &= \hat{\upsilon}_i -\kappa_q \nabla_{\hat{q}_i} \mathcal{L}(\hat{\bm{q}}),  \qquad \qquad \quad \ i \in \Vset,
        \\
        \dot{\hat{Q}}_i& = \hat{Q}_i S\big(\omega_i - \hat{Q}_i^\top \omega_a + \kappa_Q \Omega^s_i\big),  \quad i \in \Vset^s,
    \end{split}
    \label{eq:upstream}
    \\
    \dot{\hat{Q}}_i &= \hat{Q}_i S\big(\omega_i - \hat{Q}_i^\top \omega_a + \kappa_Q \Omega^f_i\big), \quad i \in \Vset^f.
    \label{eq:downstream}
\end{align}
\begin{remark}
    The observer \eqref{eq:upstream}--\eqref{eq:downstream} has a cascade structure in which the subsystem \eqref{eq:upstream} is upstream of the subsystem \eqref{eq:downstream}, which does not feed back into it. 
    This follows from the observation that $\hat{\upsilon}_i$ (\eqref{eq:vel_leader} and \eqref{eq:vel_follower}) and $\Omega_i^s$ \eqref{eq:grad_ori_leaders} depend on $\hat{\bm{q}}, \hat{\bm{Q}}^s$ only, whereas $\Omega_i^f$ \eqref{eq:grad_ori_followers} depends on $\hat{\bm{q}}, \hat{\bm{Q}}^s$, and $\hat{\bm{Q}}^f$.
    \label{rmk:cascade}
\end{remark}
\begin{remark}
    The observer \eqref{eq:upstream}--\eqref{eq:downstream} is distributed: it only requires information exchange between neighbors in the communication graph. See \eqref{eq:grad_ori_leaders}, \eqref{eq:grad_ori_followers}, and Remarks~\ref{rmk:pos_est_dist} and~\ref{rmk:vel_est_dist}.
\end{remark}

\section{Stability Analysis}
\label{sec:stability}

Define the estimation errors as $\tilde{q}_i\!\coloneqq\!\hat{q}_i-q_i$ and $\tilde{Q}_i\!\coloneqq\!Q_i^\top \hat{Q}_i$, and let $\tilde{\bm{q}}\!\coloneqq\![\tilde{q}_i]_{i\in\Vset}$, $\tilde{\bm{Q}}\!\coloneqq\![\tilde{Q}_i]_{i\in\Vset}$.
Let $\bm{\eta}  \coloneqq [\eta_i^\ast]_{i \in \Vset} \in \R^{3N}$ and $\bm{\Omega} \coloneqq [\Omega_i^\ast]_{i \in \Vset}$ where $\ast \in \{s, f\}$.
Also, for simplicity, we denote $\bm{\zeta}(\hat{\bm{q}}) \coloneqq \nabla_{\hat{\bm{q}}} \mathcal{L}(\hat{\bm{q}})$.
Thus, the error dynamics of \eqref{eq:upstream}--\eqref{eq:downstream} are
\begin{align}
    \dot{\tilde{\bm{q}}} &= \bm{\eta}(\hat{\bm{q}}, \hat{\bm{Q}}^s) - \bm{\eta}(\bm{q}, \bm{Q}^s) -\kappa_q \bm{\zeta}(\hat{\bm{q}}), 
    \label{eq:error_dynamics_pos}
    \\
    \dot{\tilde{\bm{Q}}}^s &= {-\bm{S}(\bm{\omega}^s) \circ \tilde{\bm{Q}}^s} + \tilde{\bm{Q}}^s \circ \bm{S}({\bm{\omega}^s + }\kappa_Q \bm{\Omega}^s(\hat{\bm{q}}, \hat{\bm{Q}}^s)),
    \label{eq:error_dynamics_ori_leaders}
    \\
    \dot{\tilde{\bm{Q}}}^f &= {-\bm{S}(\bm{\omega}^f) \circ \tilde{\bm{Q}}^f} + \tilde{\bm{Q}}^f \circ \bm{S}({\bm{\omega}^f + }\kappa_Q \bm{\Omega}^f(\hat{\bm{q}}, \hat{\bm{Q}})).
    \label{eq:error_dynamics_ori_followers}
\end{align}

\begin{lemma}
\label{lem:equilibrium_2}
The zero-error trajectory $(\tilde{\bm{q}}, \tilde{\bm{Q}})= \bm{1}_N \otimes (\bm{0}, \bm{I}_3)$  is invariant for the error dynamics \eqref{eq:error_dynamics_pos}--\eqref{eq:error_dynamics_ori_followers}.
\end{lemma}
\begin{proof}
    At zero error, one has $\bm{\zeta}(\bm{q})\!=\!0$~\eqref{eq:position_gradient_i}, $\bm{\Omega}^s(\bm{q}, \bm{Q}^s) = 0$~\eqref{eq:grad_ori_leaders}, and $\bm{\Omega}^f(\bm{q}, \bm{Q}) = 0$~\eqref{eq:grad_ori_followers}; hence $(\dot{\tilde{\bm{q}}}, \dot{\tilde{\bm{Q}}}) = 0$.
\end{proof}

Before establishing the main theorem, we define small orientation errors in local coordinates $\xi_i\in\R^3$, $i \in \Vset$, such that $\tilde{Q}_i = \bm{I} + S(\xi_i) + o(\|\xi_i\|)$.
Then, let  $\bm{\xi}^s \coloneqq [\xi_i]_{i \in \Vset^s}$.

\begin{theorem}[Stability of the full pose observer]
\label{thm:full_cascade_lues}
    Let $(\bm{q}(t), \bm{Q}{(t)})\in\SE^N$, $t \geq 0$, satisfy Assumptions \ref{ass:separation}--{\ref{ass:exact_consensus}}. 
    Consider the observer \eqref{eq:upstream}--\eqref{eq:downstream} and fix the orientation gain $\kappa_Q$.
    Then, if the position gain $\kappa_q\!>\!0$ is sufficiently large, the tracking error with respect to the true trajectory $(\bm{q}(t), \bm{Q}{(t)})$ is locally uniformly exponentially stable (LUES).
\end{theorem}
\begin{proof}
    \textit{Step 1:}
    We begin by establishing the LUES property of the upstream observer  \eqref{eq:upstream}.
    The linearized error dynamics \eqref{eq:error_dynamics_pos}--\eqref{eq:error_dynamics_ori_leaders} near $(\tilde{\bm{q}}, \tilde{\bm{Q}}) = (\bm{0}, \bm{I})$ are
    \begin{equation*}
        \begin{bmatrix}
            \dot{\tilde{\bm{q}}} \\ \dot{\bm{\xi}}^s
        \end{bmatrix}
        =
        \begin{bmatrix}
            \bm{C} - \kappa_q \bm{H} & \bm{D} \\
            - \kappa_Q \bm{B} & {- \bm{W}} - \kappa_Q \bm{K}
        \end{bmatrix}
        \begin{bmatrix}
            \tilde{\bm{q}} \\ \bm{\xi}^s
        \end{bmatrix} + o(\|[\tilde{\bm{q}}; \bm{\xi}^s]\|).
    \end{equation*}
    Here, $\bm{H}(t)$ is the position Hessian, which is uniformly positive definite by Lemma \ref{lem:positive_hessian}.
    $\bm{C}(t), \bm{D}(t)$ are the jacobians of $\bm{\eta}(\bm{q}, \bm{Q}^s)$ with respect to $\bm{q}$ and the local coordinates of $\bm{Q}^s$ evaluated at $\bm{q}(t)$ and $\bm{Q}^s(t)$, respectively.
    Both $\bm{C}(t), \bm{D}(t)$ are uniformly bounded since the maps $\eta^s, \eta^f$ are uniformly smooth in a neighborhood of the true trajectory.

    Next, we derive $\bm{K}(t)$, $\bm{B}(t)$, and $\bm{W}(t)$.
    From \eqref{eq:grad_ori_leaders},
    \begin{align*}
        \Omega^s_i &= \textstyle\sum_{j \in \Oset_i} S(\beta_{ij}) \tilde{Q}_i^\top (\|q_{ij}\| \beta_{ij} + Q_i^\top \tilde{q}_{ij}) =
        \\ 
        & = \textstyle\sum_{j \in \Oset_i} \left(S(\beta_{ij}) Q_i^\top \tilde{q}_{ij} - \|q_{ij}\| \Pi(\beta_{ij}) \xi_i\right) + o(\|[\tilde{\bm{q}}; \xi_i]\|). 
    \end{align*}
    where we used $S(\beta_{ij})S(\xi_i)\beta_{ij} =
    \Pi(\beta_{ij})\xi_i$.
    Now, the linearized error dynamics of \eqref{eq:error_dynamics_ori_leaders} satisfy
    \begin{align*}
        \dot{\xi}_i & = {-S(\omega_i) \xi_i} + \kappa_Q \Omega^s_i +o(\|\xi_i\| + \|\Omega^s_i\|)
        \\
        & = {-S(\omega_i) \xi_i} + \kappa_Q \left(\textstyle\sum_{j \in \Oset_i}\!B_{ij} \tilde{q}_{ij} - K_i \xi_i\right) + o(\|[\tilde{\bm{q}}; \xi_i]\|),
    \end{align*}
    where $K_i \coloneqq \textstyle\sum_{j\in\mathcal O_i} \|q_{ij}\| \Pi(\beta_{ij})$ and $B_{ij} \coloneqq S(\beta_{ij}) Q_i^\top$.
    It follows that $\bm{K}(t)\coloneqq\operatorname{diag}(K_i)_{i\in\Vset^s}$ is uniformly positive definite, from Assumptions \ref{ass:separation} and \ref{ass:noncollinear_outgoing_bearings}.
    Also, $\bm{B}(t)$ is such that block $[\bm{B}]_{ii} = \sum_{j \in \Oset_i} B_{ij}$, $[\bm{B}]_{ij} = -B_{ij}$ if $j \in \Oset_i$, and zero elsewhere. 
    Then, $\bm{B}(t)$ is uniformly bounded since it only depends on unit-norm bearings and rotation matrices.
    {Finally, $\bm{W}(t) \coloneqq \operatorname{diag}(S(\omega_i))_{i \in \Vset^s}$ is skew-symmetric and uniformly bounded.}

    Hence, the involved matrices satisfy uniformly for $t \geq 0$, $\bm{H}(t) \succeq \ell_h \bm{I}$, $\bm{K}(t) \succeq \ell_k \bm{I}$, $\|\bm{B}(t)\| \leq \mu_b$, $\|\bm{C}(t)\| \leq \mu_c$, $\|\bm{D}(t)\| \leq \mu_d$, {$\|\bm{W}(t)\| \leq \mu_w$} for some $\ell_h, \ell_k >0$ and $\mu_b, \mu_c, \mu_d, \mu_w \geq 0$.
    Now, let $\bm{e}\coloneqq [\tilde{\bm{q}}; \bm{\xi}^s]$ denote the upstream error, and consider the Lyapunov function $V(\bm{e}) \coloneqq \frac{1}{2} \|\bm{e}\|^2$.
    Then,
    \begin{equation}
    \begin{split}
        \!\!\dot{V}(\bm{e})
        & = \tilde{\bm{q}}^\top (\bm{C} - \kappa_q \bm{H}) \tilde{\bm{q}} + \tilde{\bm{q}}^\top (\bm{D} - \kappa_Q \bm{B}^\top) \bm{\xi}^s 
        \\
        -& \kappa_Q  (\bm{\xi}^s)^\top \bm{K} \bm{\xi}^s + o(\|\bm{e}\|^2) \leq (\mu_c - \kappa_q \ell_h) \|\tilde{\bm{q}}\|^2
        \\
        +& (\mu_d + \kappa_Q \mu_b) \|\tilde{\bm{q}}\| \|\bm{\xi}^s\| 
        - \kappa_Q \ell_k \|\bm{\xi}^s\|^2 + o(\|\bm{e}\|^2).
        \label{eq:lyapunov_derivative}
    \end{split}
    \end{equation}
    On the right-hand side of \eqref{eq:lyapunov_derivative}, the higher-order terms are dominated by the quadratic term, which, according to Sylvester's criterion, is negative definite if 
    $\kappa_q \ell_h  > \mu_c + {(\mu_d + \kappa_Q \mu_b)^2}/({4 \kappa_Q \ell_k})$.
    Hence, for sufficiently small errors and for sufficiently large $\kappa_q$, there exists $\rho > 0$ such that $\dot{V} \leq - \rho\|\bm{e}\|^2 = -2 \rho V$.
    It follows that $V(t )\leq V(0) e^{-2\rho t}$, thus $\|\bm{e}(t)\| \leq C_e e^{-\lambda_e t}  \|\bm{e}(0)\|$ for some constants $C_e \geq 1$ and $\lambda_e > 0$.
    Because all constants in the estimates are uniform over time, the upstream error is LUES.

    \textit{Step 2:} Now, we consider the error dynamics \eqref{eq:error_dynamics_ori_followers} with an ideal upstream observer, i.e., $\hat{\bm{q}} = \bm{q}$ and $\hat{\bm{Q}}^s = \bm{Q}^s$.
    We prove the result for each $i \in \Vset^f$ separately since each subsystem is independent.
    For each $i\in\Vset^f$ the reconstructed velocity is exact, that is $\hat{\upsilon}_i = \dot{q}_i$, and one gets $\hat{\gamma}_i = Q_i u_i$.
    Then, from \eqref{eq:grad_ori_followers}, $\Omega_i^f = S(u_i)\tilde{Q}_i^\top u_i$.
    We now linearize \eqref{eq:error_dynamics_ori_followers} around $\tilde{Q}_i = \bm{I}$.
    It follows that $\Omega_i^f
    = -S(u_i)^\top S(u_i) \xi_i + o(\|\xi_i\|)$.
    Thus, the linearized orientation error of \eqref{eq:error_dynamics_ori_followers} satisfies
    \begin{equation*}
        \dot{\xi}_i = - G_i(t) \xi_i + o(\|\xi_i\|), \ G_i \coloneqq {S(\omega_i)} + \kappa_Q S(u_i)^\top S(u_i).
    \end{equation*}
    Now, fix an initial time $t \geq 0$, and let $\tau \in [t, t+T]$.
    Consider the change of coordinates $\phi_i(\tau) \coloneqq \Phi_i(\tau) \xi_i(\tau)$, where $\Phi_i(\tau) \coloneqq R_i(t)^\top R_i(\tau) \in \SO$.
    Here, $\Phi_i(\tau)$ is used to discount the transport of the error coordinates induced by $\omega_i$.
    Since $\dot{p}_i(\tau) = R_i(\tau) u_i(\tau)$, it follows that
    \begin{align*}
        \tfrac{d\phi_i}{d\tau} 
        =\!- \kappa_Q R_i^\top(t) S(\dot{p}_i(\tau))^\top S(\dot{p}_i(\tau)) R_i(t) \phi_i(\tau) + o(\|\phi_i\|).
    \end{align*}
    If Assumption \ref{ass:excitation} is satisfied over $[t, t+T]$, that is, $\dot{p}_i$ is PE, then this linear system is UES.

    \textit{Step 3:} 
    Now, we consider the error dynamics \eqref{eq:error_dynamics_ori_followers} when the upstream error $\bm{e}(t)$ is nonzero but exponentially decaying, as established by \textit{Step 1}.
    Consider $i\in\Vset^f$, and let
    \begin{equation*}
        r_i\!\coloneqq\!\hat{\upsilon}_i - \dot{q}_i =\eta^f_i(\hat{\bm{q}},\hat{\bm{Q}}^s) - \eta^f_i(\bm{q},\bm{Q}^s)
        = C_i \tilde{\bm{q}} + D_i \bm{\xi}^s + o(\|\bm{e}\|),
    \end{equation*}
    where $C_i(t), D_i(t)$ are the jacobian matrices of $\eta^f_i$ with respect to $\bm{q}$ and the local coordinates of $\bm{Q}^s$ evaluated at $\bm{q}(t)$ and $\bm{Q}^s(t)$, respectively.
    These matrices are uniformly bounded due to the uniform smoothness of the map $\eta^f_i$.
    Thus, there exists $\rho_q, \rho_{\xi}, \rho_e, \rho_0, \lambda_e>0$ such that
    \begin{equation*}
    \|r_i(t)\| \leq \rho_q \|\tilde{\bm{q}}(t)\| + \rho_{\xi} \|\bm{\xi}^s(t)\| \leq \rho_e\|\bm{e}(t)\| \leq \rho_0 e^{-\lambda_e t}.
    \end{equation*}
    On the other hand, it holds that $\hat{\gamma}_i = Q_i u_i + s_i, $
    where $s_i \coloneqq r_i + {S(\omega_a) \tilde{q}_i}$ decays exponentially.
    Thus, from \eqref{eq:grad_ori_followers},
    \begin{equation*}
        \Omega_i^f =  S(u_i)\tilde{Q}_i^\top u_i + S(u_i) \tilde{Q}_i^\top Q_i^\top s_i.
    \end{equation*}
    Then, the linearized orientation error of \eqref{eq:error_dynamics_ori_followers} satisfies
    \begin{equation*}
        \dot{\xi}_i = -G_i(t)\xi_i + \kappa_Q S(u_i) Q_i^\top s_i + o(\|[\bm{e}; \xi_i]\|).
    \end{equation*}
    The unperturbed linear system ($s_i\!=\!0$) is UES as shown in \textit{Step 2}.
    Since $s_i(t)$ decays exponentially, a variation-of-constants argument shows that $\xi_i(t)$ also decays exponentially.
    Hence, the nonlinear error dynamics are LUES near $\tilde{Q}_i=\bm{I}$.
    Because the downstream system does not feed back into the upstream subsystem, the full observer is LUES.
\end{proof}

\section{Simulation Results}
\label{sec:simulations}

To validate the proposed estimator, we conducted simulations with $N = 5$ robots in $\mathbb{R}^3$ and the anchor robot was $a=1$.
In the first example (Fig. \ref{fig:simu_1}), the sensing graph was determined by $\Oset_1 = \{2, 3, 4, 5\}$, $\Oset_2 = \{1, 3, 4, 5\}$, $\Oset_3 = \Oset_4 = \Oset_5 = \emptyset$.
So, $\Vset^s = \{1, 2\}$ and $\Vset^f = \{3, 4, 5\}$.

Let $\mathcal{U}(\cdot)$ and $\mathcal{N}(\cdot, \cdot)$ denote the uniform and normal distributions, respectively.
The initial robots' positions were generated as $q_a(0)\!\!=\!\!0$ and $q_i(0)\!\sim\!\mathcal{U}([\SI{0}{\meter}, \SI{20}{\meter}]^3)$.
Orientations were sampled on $\SO$ by setting $Q_a(0)\!=\!\bm{I}$ and $Q_i(0)\!=\!\exp(S(\theta_i))$, with $\theta_i\sim\mathcal{U}(\{x\!:\!\|x\|\!<\!2\pi\,\mathrm{rad}\})$.
The initial position estimates were $\hat{q}_a(0) = 0$ and $\hat{q}_i(0) \sim \mathcal{N}(q_i(0), \SI{4}{\square\meter})$.
The initial orientation estimates were $\hat{Q}_a(0) = \bm{I}$ and $\hat{Q}_i(0) = Q_i(0) \exp(S(\psi_i))$, where $\psi_i\sim\mathcal{U}(\{x\!:\!\|x\|\!<1\,\mathrm{rad}\})$.
The body-frame control actions were $u_1(t) = [0; 0; 1]$, $u_2(t) = [\cos(t/4); \sin(t/4); 0]$, $u_3(t) = [0; \cos(t); \sin(t)]$, $u_4(t) = [\cos(2t); \sin(2t); 0.5]$, and $u_5(t) = [\cos(t); \sin(0.5t); 0.0]$.
The body-frame angular velocities were $\omega_1 = 0.5 \mathfrak{e}_1$, $\omega_2 = \mathfrak{e}_2$, $\omega_3 = \mathfrak{e}_3$, and $\omega_4 = \omega_5 = 0$.
The resulting trajectories are shown in Fig. \ref{fig:simu_1_traj}, where the body-frame axes are omitted for simplicity.
Note that $\dot{p}_3, \dot{p}_4, \dot{p}_5$ are persistently exciting, as per Assumption \ref{ass:excitation}.
Fig \ref{fig:simu_1_err} shows the errors obtained by the proposed observer with gains $\kappa_s = 0.02$, $\kappa_q = 500$, and $\kappa_Q = 2$.
For the orientation error, we plot $\|\psi_i\|= \arccos((\operatorname{tr}(\tilde{Q}_i) - 1)/2)$ where $\tilde{Q}_i = \exp(S(\psi_i))$.
As expected, the errors decay toward zero, demonstrating the effectiveness of the proposed estimator.
Notably, the sensing robots' orientation errors diminish faster than those of the free robots, who rely on error correction over a period rather than instantaneously.

A second case was considered where we set $\omega_i(t) = 0$ and $u_i(t) = \hat{Q}_i^\top(t) [\cos(t); \sin(t/2); 0]$; see Fig. \ref{fig:simu_2}.
As $\hat{Q}_i \to Q_i$ robots converge to a common velocity.
This is done to demonstrate that no relative motion between the robots is required for convergence.
Here, we employed a different sensing graph having $\Oset_1 = \{2, 4, 5\}$, $\Oset_2 = \{1, 3, 4\}$, $\Oset_3 = \{1, 2, 5\}$, and $\Oset_4 = \Oset_5 = \emptyset$.
Hence, $\Vset^s = \{1, 2, 3\}$ and $\Vset^f = \{4, 5\}$.
Fig. \ref{fig:simu_2_err} shows the errors decaying toward zero.
Two sharp decreases in $\psi_i$, $i \in \Vset^f$, can be observed, beginning at $t \approx \SI{1.5}{\second}$ and $t \approx \SI{5}{\second}$.
These instants coincide with a change in motion direction, illustrating the PE condition.
Note that the sensing graphs employed in these examples are IAR; however, they are neither IBR nor compatible with the sensing structures required in \cite{VanTran2020AUT,Cao2021AUT,Boughellaba2022CDC,Boughellaba2023CSL}.
This underscores the method’s versatility for sparse, unstructured sensing networks found in practice.

\begin{figure}
    \centering
    \begin{subfigure}{0.49\columnwidth}
        \includegraphics[width=\columnwidth]{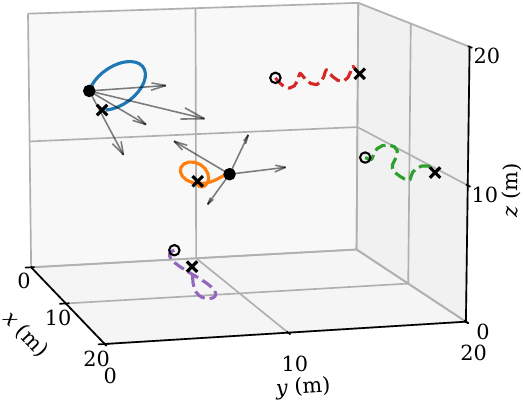}
        \caption{Robot trajectories.}
        \label{fig:simu_1_traj}
    \end{subfigure}
    \begin{subfigure}{0.49\columnwidth}
        \includegraphics[width=\columnwidth]{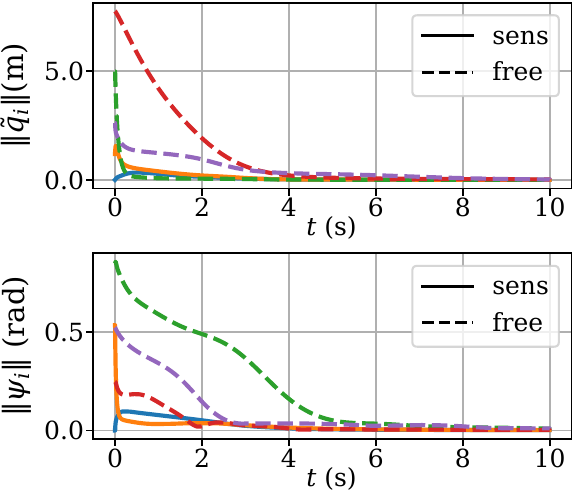}
        \caption{Pose error.}
        \label{fig:simu_1_err}
    \end{subfigure}
    \caption{Simulation results: Case 1. (a): $\bullet$ and $\circ$ represent initial position of $\Vset^s$ and $\Vset^f$, respectively; $\times$ are the final positions; and arrows show the sensing graph.
    }
    \label{fig:simu_1}
\end{figure}

\begin{figure}
    \centering
    \begin{subfigure}{0.49\columnwidth}
        \includegraphics[width=\columnwidth]{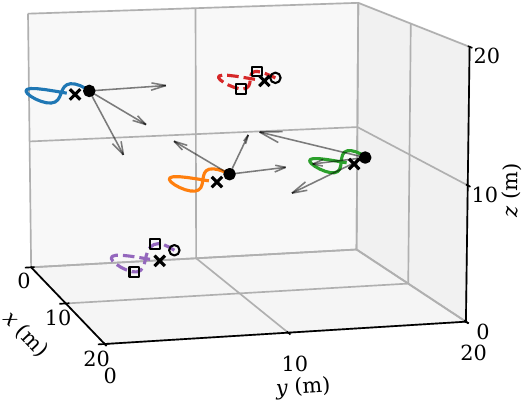}
        \caption{Robot trajectories.}
        \label{fig:simu_2_traj}
    \end{subfigure}
    \begin{subfigure}{0.49\columnwidth}
        \includegraphics[width=\columnwidth]{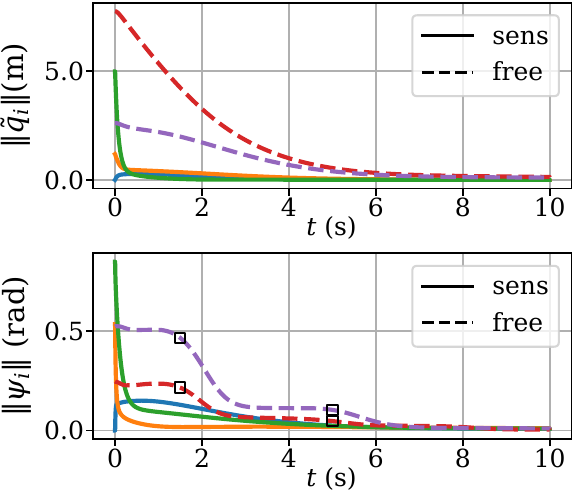}
        \caption{Pose error.}
        \label{fig:simu_2_err}
    \end{subfigure}
    \caption{Simulation results: Case 2. (a)-(b): {\tiny $\square$} marks the instants $t = \SI{1.5}{\second}$ and $t = \SI{5}{\second}$.
    }
    \label{fig:simu_2}
\end{figure}

\section{Conclusions}
\label{sec:conclusions}

This work presents a novel observer for distributed pose estimation in multi-robot systems.
By exploiting IAR and PE motions, the usual requirement that all robots must acquire two or more bearings is relaxed. 
Local exponential stability of the poses was established under PE motions for the free robots. Notably, all positions, as well as the orientations of sensing robots, can be recovered when all robots are static.

In future work, we will analyze the coupled consensus-observer dynamics and characterize the resulting region of attraction. 
We will also extend the stability analysis to encompass time-varying sensing networks.
Finally, we will assess robustness to measurement noise, including noise in bearing and bearing-rate measurements.

\bibliographystyle{IEEEtran}
\bibliography{refs}

\end{document}